\documentclass{article}
\usepackage{sp_conf}
\usepackage{cite}
\usepackage{amsmath,amssymb,amsfonts}
\usepackage{graphicx}
\usepackage{textcomp}
\usepackage{xcolor}
\def\BibTeX{{\rm B\kern-.05em{\sc i\kern-.025em b}\kern-.08em
    T\kern-.1667em\lower.7ex\hbox{E}\kern-.125emX}}

\usepackage[english]{babel}
\usepackage[utf8]{inputenc}
\usepackage[T1]{fontenc}
\usepackage{lmodern}
\usepackage{braket}
\usepackage{wrapfig}
\usepackage{graphicx}
\usepackage{amsmath,amssymb,amsfonts,verbatim}
\usepackage{hyperref}
\usepackage{mathtools}
\usepackage{enumitem}

\usepackage{xr}

\usepackage{amsmath,amssymb,amsfonts,verbatim}
\usepackage{graphicx}
\usepackage[utf8]{inputenc}
\usepackage{textcomp}
\usepackage{xcolor}
\usepackage{braket}
\usepackage{soul}
\usepackage[colorinlistoftodos]{todonotes}

\usepackage{caption}
\usepackage{subcaption}

\usepackage{enumitem}

\usepackage{cite}
\usepackage{amsmath,amssymb,amsfonts,verbatim}
\usepackage{amsthm}
\usepackage{graphicx}
\usepackage[utf8]{inputenc}
\usepackage{textcomp}
\usepackage{xcolor}
\usepackage{braket}
\usepackage{soul}
\usepackage{enumitem}
\usepackage[colorinlistoftodos]{todonotes}

\usepackage{algpseudocode}
\usepackage{algorithm}

\usepackage{wrapfig}

\usepackage{caption}

\newtheorem{theorem}{Theorem}

\newtheorem{corollary}{Corollary}

\newtheorem{definition}{Definition}

\newtheorem{assumption}{Assumption}

\usepackage{tikz}
\usetikzlibrary{positioning}
 
\tikzstyle{pinstyle} = [pin edge={to-,thin,black}]
 
\allowdisplaybreaks

\algnewcommand{\IIf}[1]{\State\algorithmicif\ #1\ \algorithmicthen}
\algnewcommand{\EndIIf}{\unskip\ \algorithmicend\ \algorithmicif}

\newcommand{\reals}{\mathbb{R}}

\newcommand{\dataset}{\mathcal{D}}
\newcommand{\ndataset}{\hat{\mathcal{D}}}

\newcommand{\CX}{\mathcal{X}}

\newcommand{\na}{M}

\newcommand{\lconst}{\alpha_{\dtime}}

\newcommand{\dtime}{t}
\newcommand{\respi}{\var^i_{\dtime}}

\newcommand{\sw}{\mu}

\newcommand{\var}{\beta}
\newcommand{\varti}{\var_t^i}

\newcommand{\dimn}{N}

\newcommand{\measi}{y^i}

\newcommand{\sdim}{q} 
\newcommand{\mdim}{p} 

\newcommand{\argo}{\beta}

\newcommand{\hu}{\hat{u}}
\newcommand{\hlam}{\hat{\lambda}}
\newcommand{\hb}{\hat{\beta}}

\newcommand{\wass}{\mathcal{W}}

\newcommand{\CE}{\mathbb{E}}
\newcommand{\CV}{\mathcal{V}}
\newcommand{\bv}{\boldsymbol{v}}

\newcommand{\util}{f}

\title{Distributionally Robust Inverse Reinforcement Learning for Identifying Multi-Agent Coordinated Sensing
}

\name{Luke~Snow,
        Vikram~Krishnamurthy
\thanks{This research was funded by National Science Foundation grant CCF-2112457,  Army Research office grant W911NF-21-1-0093 , and Air Force Office of Scientific Research grant FA9550-22-1-0016.}}
\address{Electrical and Computer Engineering, Cornell University, Ithaca, NY}

\begin{document}




\maketitle

\begin{abstract}
We derive a minimax distributionally robust inverse reinforcement learning (IRL) algorithm to reconstruct the utility functions of a multi-agent sensing system. Specifically, we construct utility estimators which minimize the worst-case prediction error over a Wasserstein ambiguity set centered at noisy signal observations. We prove the equivalence between this robust estimation and a semi-infinite optimization reformulation, and we propose a consistent algorithm to compute solutions.
We illustrate the efficacy of this robust IRL scheme in  numerical studies to reconstruct the utility functions of a cognitive radar network from observed tracking signals.
\end{abstract}
\begin{keywords}
Distributionally Robust Optimization, Multi-Agent Inverse Reinforcement Learning, Revealed Preferences, Wasserstein Distance
\end{keywords}
\vspace{-0.4cm}
\begin{section}{\small Introduction}
\vspace{-0.3cm}

How to identify  if a multiagent system is making decisions consistent with Pareto optimality (we call this "coordination"), and then reconstruct the utility functions of individual agents? This problem is referred to as multi-agent inverse reinforcement learning (IRL) \cite{natarajan2010multi}, \cite{yu2019multi}, in machine learning or collective revealed preferences in microeconomics \cite{cherchye2011revealed}, \cite{nobibon2016revealed}.  Recent works \cite{krishnamurthy2020identifying}, \cite{pattanayak2022meta}, \cite{snow2023statistical}, \cite{krishnamurthy2021adversarial}, explore the use of  IRL in cognitive sensing applications. 

This paper addresses the problem of \textit{robust} multiagent IRL when the system's decisions are observed in noise.  Motivated by recent results in distributionally robust optimization  \cite{dong2021wasserstein}, \cite{luo2017decomposition}, \cite{bertsimas2019adaptive}, we devise a robust multiagent IRL algorithm using revealed preferences. Specifically, we propose an algorithm that constructs utility functions in a minimax sense;  minimize the maximum reconstruction error within a Wasserstein ambiguity set centered at the noisy observed signals. This extends works in stochastic revealed preferences \cite{mcfadden2006revealed}, \cite{bandyopadhyay1999stochastic}, \cite{aguiar2021stochastic}. 

{\bf Context}. In summary, we study robust {\em inverse} multiobjective optimization (subject to sensing constraints) when the optimizers are observed in noise, using revealed preferences. While this paper focuses on the underlying theory and algorithms, our main motivation stems from multi-agent IRL in radar or drone networks.
Inverse optimization is an ill-posed problem; so we focus on set valued reconstruction of the utility.

{\bf Main Results}. 
 We provide a framework for multi-agent sensor system utility reconstruction from noisy observed sensing signals, extending the techniques in \cite{snow2023statistical}, \cite{snow2022identifying}, \cite{krishnamurthy2020identifying}.
We then derive a Wasserstein-distributionally robust utility reconstruction objective, and prove its equivalence to a semi-infinite program reformulation.
We provide a finite reduction of this semi-infinite program and a practical algorithm for achieving a $\delta$-optimal solution.
We illustrate the efficacy of this robust reconstruction algorithm via numerical simulations. 

\end{section}
\vspace{-0.4cm}
\begin{section}{\small Coordinated Sensing Systems}
\label{sec:radarnet}
\vspace{-0.4cm}
We consider the interaction between a stochastic dynamical system ("target") and a sensing system comprising $\na$ heterogeneous sensors. The target evolves according to a state-space model, and each of the $\na$ sensors records noisy observations of the target's state.

\vspace{-0.2cm}
\begin{definition}[Multi-agent Bayesian Sensing System]
\label{inter_dynam}
We introduce the following state-space sensing dynamics:
\vspace{-0.3cm}
\begin{align*}
    \begin{split}
        \textrm{target state} : x_t &\in \reals^{\sdim}, \
        x_{t + 1} \sim p_{\lconst}(x| x_t) \\
        \textrm{state dynamics parameter}: \lconst &\in \reals^{\dimn}_+ \\
        \textrm{sensor i observation}: \measi_{t} &\in \reals^{\mdim}, \
        \measi_t \sim p_{\varti}(y|x_t) \\
        \textrm{sensor i parameter} : \varti &\in \reals^{\dimn}_+, \, i\in[\na]\\
    \end{split}
\end{align*}
\end{definition}
\noindent 
$[x]$ denotes the set $\{1,\dots,x\}$. Each sensor $i$ has utility function $\util^i: \reals^{\dimn}_{+} \mapsto \reals$, quantifying its sensing objective, and may adjust its sensing mechanism through parameter $\beta_t^i$ (e.g., its tracking signal power or waveform) to achieve its objective. In a \textit{coordinated} sensing system, the individual sensing mechanisms (we identify these with \textit{signal outputs}) $\beta_t^i$ are coupled, so the group outputs signals which maximize the aggregate utility:
\vspace{-0.1cm}
\begin{definition}[Coordinated Sensing System]
\label{def:coord}
 Consider Def.~\ref{inter_dynam}. We define a coordinating sensing system to be a group of $\na$ sensors, each with individual concave, continuous and monotone increasing objective functions $\util^i: \reals^{\dimn} \to \reals, i\in[\na]$, which produces output signals $\{\varti\}_{i=1}^{\na}$ in accordance with\footnote{The constraint bound 1 is without loss of generality, see \cite{krishnamurthy2020identifying}.}
 \vspace{-0.3cm}
 \begin{gather}
\begin{aligned}
\label{def:coord_eq}
    \begin{split}
        \{\varti\}_{i=1}^{\na} \in \arg\max_{\{\argo^i\}_{i=1}^{\na}} \sum_{i=1}^{\na}\mu^i\util^i(\argo^i) \,\, s.t. \,\, \lconst' (\sum_{i=1}^{\na} \argo^i) \leq 1
    \end{split}
\end{aligned}\raisetag{2.4\baselineskip}\end{gather}
for a set of weights $\mu^i > 0$.
\end{definition}

\vspace{-0.2cm}
A group which emits signals according to \eqref{def:coord_eq} optimally (in the \textit{Pareto} sense) parameterizes the measurement kernels $p_{\beta_t^i}(y|x_t)$ subject to each objective function, the state dynamics of the target, and a constraint on the sensing accuracy (e.g., total power output). Due to space constraints, we do not motivate this further: see \cite{snow2022identifying}, \cite{snow2023statistical} for details on how the constrained multi-objective optimization \eqref{def:coord_eq}, especially the joint constraint, arises from spectral optimization within the dynamics of Def.~\ref{inter_dynam}. 

\vspace{-0.3cm}

\section{\small Coordination Detection and Utility Reconstruction} 
\label{sec:MOO_detector}
\vspace{-0.3cm}
We  take the perspective of the \textit{target/analyst}, that \textit{aims to determine if the sensing system is coordinating \eqref{def:coord_eq}, from observed sensing signals. We then \textit{aim to reconstruct utility functions giving rise to these signals}.}
 
Specifically, as the target we obtain $\{\alpha_t, t\in[T]\}$ through our own dynamics, and we observe the sensing signals $\{\varti, t\in[T]\}_{i=1}^{\na}$ through e.g., an omni-directional receiver. We denote the dataset of these signals as $\dataset = \{\lconst, \{\respi\}_{i=1}^{\na}, t \in [T] \}$. See \cite{pace2009detecting} for physical-layer considerations of sensing waveform observation, detection, and classification. Here we provide a necessary and sufficient condition for the dataset $\dataset$ to be consistent with coordination (Def~\ref{def:coord}).
\vspace{-0.2cm}
\begin{theorem}
 \label{thm:cherchye1}
    Let $\dataset$ be a set of observations. The following are equivalent:
    \begin{enumerate}
    \item there exist a set of $M$ concave and continuous objective functions $\util^1,\dots,\util^m$, weights $\mu^i > 0$ and constraint $p^*$ such that $\forall t \in [T]$:
    \vspace{-0.2cm}
    \begin{gather}\begin{aligned}
    \begin{split}
    \label{thm1:rat}
        \{\respi\}_{i=1}^{\na} \in &\arg\max_{\{\argo^i\}_{i=1}^{\na}} \sum_{i=1}^{\na} \sw^i \util^i(\argo^i) \ \, s.t.\ \lconst' (\sum_{i=1}^{\na}\argo^i ) \leq 1
    \end{split}
    \end{aligned}\raisetag{2.5\baselineskip}\end{gather}
    \vspace{-0.4cm}
    \item there exist numbers $u_j^i \in \reals, \lambda_j^i > 0$ such that for all $s,t \in [T]$, $i \in [M]$: 
     \vspace{-0.2cm}
    \begin{equation}
    \label{af_ineq}
        u_s^i - u_t^i - \lambda_t^i\lconst'[\var_s^i - \varti] \leq 0
        \vspace{-0.4cm}
    \end{equation}
    \end{enumerate}
\end{theorem}
\vspace{-0.2cm}
\begin{proof}
See Theorem 1 of \cite{snow2022identifying}
\end{proof}
\vspace{-0.3cm}
Thus, we can simply solve the linear program \eqref{af_ineq} feasibility to test for coordination in the sensing system. Then given feasibility (coordination), we can use the following Corollary to reconstruct utility functions which rationalize the observed signals. 

\vspace{-0.2cm}
\begin{corollary}
\label{cor:Utrec}
Given constants $u_t^i, \lambda_t^i, t\in[T],i\in[M]$ which make \eqref{af_ineq} feasible, construct
\vspace{-0.2cm}
\begin{equation}
\label{eq:Utrec}
     \util^i(\cdot) = \min_{t \in [T]} \left[u_t^i + \lambda_t^i\lconst'[\cdot - \respi] \right]
     \vspace{-0.2cm}
\end{equation}
Then \eqref{thm1:rat} is satisfied with $\dataset$ and objective functions \eqref{eq:Utrec}.
\end{corollary}
 \vspace{-0.2cm}
\begin{proof}
See Lemma 1 of \cite{snow2022identifying}.
\end{proof}

\textit{Corollary~\ref{cor:Utrec} is the key tool we will expand upon in this paper}. \cite{snow2022identifying}, \cite{snow2023statistical}, and \cite{krishnamurthy2020identifying} have investigated the usage of Corollary~\ref{cor:Utrec} for reconstruction of utility functions which rationalize observed sensing signals. However, Corollary~\ref{cor:Utrec} is fundamentally limited to the deterministic regime, i.e., it does not offer guarantees on the rationalizability of a \textit{noisy} dataset. 
Next we introduce an augmentation of \eqref{eq:Utrec} for reconstructing utility functions given noisy signals, quantify the reconstruction accuracy in this case, and extend this to a \textit{distributionally robust} utility estimation procedure.

\end{section}
\vspace{-0.3cm}
\begin{section}{\small Main Result I. Robust Utility Estimation}
\label{sec:robust}
\vspace{-0.3cm}
Here we extend the utility reconstruction technique \eqref{eq:Utrec} to the noisy data regime, and provide a distributionally robust methodology for reconstructing utility functions. 
\vspace{-0.8cm}
    \subsection{\small Quantifying the Proximity to Optimality}
    \vspace{-0.2cm}
    Suppose we obtain a dataset of probes $\alpha_t$ and noisy signals $\hat{\beta}_t^i = \beta_t^i + \epsilon_t^i$, where $\epsilon_t^i$ is additive noise. Denote this noisy dataset as 
    \vspace{-0.3cm}
    \begin{equation}
    \label{eq:ndataset}
        \ndataset := \{\alpha_t,\hat{\beta}_t^i, t\in[T]\}_{i\in[\na]}
        \vspace{-0.3cm}
    \end{equation}
    We construct the following function $\phi$ acting on $\ndataset$: 
    \vspace{-0.2cm}
    \begin{align}
    \begin{split}
    \label{eq:phisolve}
        \phi(\ndataset) = &\arg\min_{r}: \exists \{u_t^i \in \reals,\lambda_t^i > 0, t\in[T]\}_{i\in[\na]}: \\&u_s^i - u_t^i - \lambda_t^i \alpha_t'[\hat{\beta}_s^i - \hat{\beta}_t^i] \leq \lambda_t^i \, r \quad \forall t,s,i
        \vspace{-0.2cm}
    \end{split}
    \end{align}
    If $\phi(\ndataset) \leq 0$ then, by Theorem~\ref{thm:cherchye1}, the dataset $\ndataset$ is consistent with coordination, and utility functions rationalizing $\ndataset$ can be constructed as \eqref{eq:Utrec}. However, given the noise in $\ndataset$ it is likely that $\phi(\ndataset) > 0$, meaning there do not exist utility functions rationalizing $\ndataset$; but in this case $\phi(\ndataset)$ represents the \textit{proximity} to consistency with \eqref{thm1:rat}, or "optimality". \cite{snow2024adaptive} provides more motivation for the construction \eqref{eq:phisolve}. 
    
    In the case when $\phi(\ndataset) > 0$ and Corollary~\ref{cor:Utrec} no longer applies, how can we reconstruct utility functions which are good \textit{approximations} of those rationalizing $\dataset$? We first outline a naive approach, then propose our robust solution.
\vspace{-0.4cm}
\subsection{\small Utility Reconstruction: Naive Approach}
Suppose the \textit{true} dataset $\dataset = \{\alpha_t,\beta_t^i\, t\in[T]\}_{i\in[\na]}$ satisfies \eqref{thm1:rat}. Then, utility functions rationalizing $\dataset$ can be constructed by \eqref{eq:Utrec} using parameters 
\[\psi := [u_1^1,\lambda_1^1,\dots,u_T^{\na},\lambda_{T}^{\na}]' \in \Psi \subseteq \reals^{2TM}\] taken from \eqref{af_ineq}, where $\Psi$ denotes the space of these vectors. 

 When handling the \textit{noisy} dataset $\ndataset$, our goal is to reconstruct utility functions $\{\hat{\util}^i(\cdot)\}_{i\in[\na]}$ closely approximating these $\{\util^i(\cdot)\}_{i\in[\na]}$.
Let $\hat{\psi}$ denote the vector corresponding to the parameters $\{\hat{u}_t^i, \hat{\lambda}_t^i, t\in[T]\}_{i\in[\na]}$ such that
    \begin{equation}
        \label{eq:parhat}
        \hat{u}_s^i - \hat{u}_t^i - \hat{\lambda}_t^i\alpha_t'[\hat{\beta}_s^i - \hat{\beta}_t^i] \leq \hat{\lambda}_t^i\, \phi(\ndataset)
    \end{equation}  
    Since $\phi(\ndataset)$ represents the closest "distance" to optimality, by \eqref{eq:phisolve}, we have that the utility functions 
    \begin{equation}
    \label{eq:noisyut}
        \hat{\util}^i(\cdot) := \min_{t\in[T]}[\hu_t^i + \hlam_t^i\alpha_t'[\cdot - \hb_t^i]]
    \end{equation}
    are the best estimates for $\{\util^i\}_{i=1}^{\na}$.\footnote{This notion of estimation accuracy can be made precise by considering the Hausdorff distance between Pareto-optimal surfaces generated by $\{\hat{\util}^i\}_{i\in[\na]}$ and $\{\util^i\}_{i\in[\na]}$. This is explained in Sec.~\ref{sec:numeric}.}  
    
    However, the stochastic perturbations in $\ndataset$ may result in reconstructed utility functions \eqref{eq:noisyut} which approximate the true utility functions very poorly in some cases, even if on average this approximation is acceptable. In particular we have no control over the \textit{worst-case} approximation, which is necessary to control in many applications \cite{gabrel2014recent}, \cite{lin2022distributionally}; this can be addressed using \textit{robust} approaches. 
    \vspace{-0.3cm}
    \subsection{\small Utility Reconstruction: Robust Approach}
    To hedge against such uncertainty arising from the choice of $\hat{\psi}$ from \eqref{eq:parhat}, we can introduce a \textit{distributionally robust utility estimation procedure}. 

    
    Let $\Phi = \{\beta_t^i, t\in[T]\}_{i\in[\na]}$ denote the dataset of signals, and $\Gamma = \otimes_{t=1}^{T}\otimes_{i=1}^{\na}\Gamma_t^i$ the domain of $\Phi$, where $\beta_t^i \in \Gamma_t^i \subseteq \reals_+^{\dimn}$. Then, a particular (noisy) instantiation $\{\hb_t^i, t\in[T]\}_{i\in[\na]}$ corresponds to the empirical distribution $P_T(\cdot) := \otimes_{t=1}^T \otimes_{i=1}^{\na} \delta(\cdot - \hb_t
    ^i)$ on $\Gamma$, where $\delta$ denotes the standard Dirac delta function on $\reals^{\dimn}$. 
    
    Let $B_{\epsilon}(P_T)$ be the set of probability distributions on $\Gamma$ with 1-Wasserstein distance at most $\epsilon$ from $P_T$.\footnote{The 1-Wasserstein distance between distributions $Q$ and $P$ on space $\CX$ is given by
    \[\wass(Q,P) = \inf_{\pi\in\Pi(Q,P)}\int_{\CX \times \CX} \| x - y\|_2 \pi(dx,dy),\]
    where $\Pi(Q,P)$ is the set of probability distributions on $\CX \times \CX$ with marginals $Q$ and $P$.} 
    
    \textit{Then, we can conceptualize the robust estimation objective as the minimax problem} 
    \begin{align}
    \begin{split}
    \label{eq:robest}
        &\min_{\psi \in \Psi}\sup_{Q \sim B_{\epsilon}(P_T)}\CE_{\Phi \sim Q}\left[ h(\psi,\Phi)\right] \\
        & h(\psi,\Phi) := \arg\min_{r}: u_s^i - u_t^i - \lambda_t^i\alpha_t'[\beta_s^i - \beta_t^i] \leq \lambda_t^i \,r \\
        &\psi = [u_1^1,\lambda_1^1,\dots,u_T^{\na},\lambda_{T}^{\na}]', \,\,\, \Phi = \{\beta_t^i,t\in[T]\}_{i\in[\na]}
    \end{split}
    \end{align}
    The objective \eqref{eq:robest} finds the set of parameters $\psi$ which minimizes the worst-case expected proximity to feasibility over possible datasets $\dataset$ with $\epsilon$ 1-Wasserstein proximity to the noisy dataset $\ndataset$. Thus, when compared to the naive estimation procedure \eqref{eq:noisyut}, \eqref{eq:robest} will better approximate the true utility functions \textit{in the worst case}, making \eqref{eq:robest} a \textit{robust} estimation procedure. 

    It remains to be shown how \eqref{eq:robest} can be computed in practice. This is the focus of the following section.
\end{section}
\begin{section}{\small Main Result II. IRL Algorithm for Robust Utility Estimation}
\label{sec:compute}
 \vspace{-0.2cm}
    Here we show the  equivalence between the distributionally robust utility estimation procedure \eqref{eq:robest} and a semi-infinite program.  We exploit this equivalence to provide a practical algorithm for computing a set of robust utility estimates.
    A semi-infinite program is an optimization problem with a finite number of variables to be optimized but an arbitrary number (continuum) of constraints.
 \vspace{-0.4cm}
    \subsection{\small  Semi-Infinite Programming Reformulation}
     \vspace{-0.2cm}
    We introduce the following assumptions and notation: 
 \vspace{-0.2cm}
    \begin{assumption}[Finite Support Noise]
    \label{as:finsup}
        The support of each additive noise $\epsilon_t^i$ distribution is contained within a ball of radius $R$. \footnote{This is satisfied in practice since any physical sensor which measures $\beta_t^i$ will have upper and lower bounds on the measured signal power.}
    \end{assumption}
 \vspace{-0.5cm}
    \begin{assumption}[Probe Magnitude Bound]
    \label{as:albd}
        $\alpha_t$ is lower bounded in magnitude: $\exists \,\bar{\alpha}: \, \|\alpha_t\| \geq \bar{\alpha} > 0 \, \forall t\in[T]$. 
    \end{assumption}
 \vspace{-0.5cm}
    \begin{assumption}[Parameter Set Bounds]
    \label{as:convex}
        There exists $\hat{\lambda}>0$ such that $\Psi$ is restricted to the set $\{[u_1^1,\lambda_1^1,\dots,u_T^{\na},\lambda_{T}^{\na}]\}$ with $u_s^i \in [-1,1],\, \lambda_s^i \in [\hat{\lambda},1], \,\forall s \in[T],i\in[\na]$.  \footnote{This is without loss of generality. Observe: if a set of parameters $\hat{\psi} = [\hu_1^1,\dots,\hlam_T^{\na}]\in \Psi$ solves \eqref{eq:parhat}, then so does $c\,\hat{\psi} := [c\hu_1^1,\dots, c\hlam_T^{\na}]$ for any scalar $c>0$. Also, given the boundedness of $\|\alpha_t\|$ and $\|\beta_t^i\|$ the ratio $\hat{u}_s^i/\hat{\lambda}_t^i$ will be bounded from above and below by positive real numbers. Thus, we can always find some $\hat{\psi}$ solving \eqref{eq:parhat} such that $\hat{u}_s^i \in [-1,1], \, \hat{\lambda}_s^i \in [\hat{\lambda},1], \forall s\in[T], i\in[\na]$, with $\hat{\lambda}>0$.}
    \end{assumption}

   \vspace{-0.2cm}
    By Assumptions~\ref{as:albd}, \ref{as:convex}, and the constraint in \eqref{thm1:rat}, we must have that $h(\psi,\Phi) \leq V := 2(1+R)+2$ for any $\psi \in \Psi, \, \Phi \in \Gamma$, with $\psi$ satisfying A~\ref{as:convex}. 
    Let us denote $\CV := \biggl\{\bv \in \reals^{2}: \,\,0 \leq v_1 \leq 2V,\, \,0\leq v_2\leq V/\epsilon\biggr\}$.
    Now, we have the following equivalence result. 
     \vspace{-0.2cm}
    \begin{theorem}[Semi-Infinite Reformulation]
        Under Assumptions~\ref{as:finsup} - \ref{as:convex}, \eqref{eq:robest} is equivalent to the following semi-infinite program:
        \begin{gather}\begin{aligned}
        \begin{split}
        \label{eq:siprog}
           & \min_{\psi\in\Psi,\bv\in\CV} \, \epsilon \cdot v_2 + v_1 \,\,\,
             s.t. \, \sup_{\Phi \in \Gamma} G(\psi,\bv, \Phi, \hat{\dataset}) \leq 0 
             \\& G(\psi,\bv, \Phi, \hat{\dataset}) := h(\psi,\Phi) - v_2\sum_{i=1}^{\na}\sum_{t=1}^{T}\|\beta_t^i - \hb_t^i \|_2 - v_1
             \vspace{-0.5cm}
        \end{split}
        \end{aligned}\raisetag{3.6\baselineskip}\end{gather}
    \end{theorem}
 \vspace{-0.4cm}
    \begin{proof}
        Under Assumptions~\ref{as:finsup}-\ref{as:convex}, $\Gamma$ and $\Psi$ are compact. We have observed that $h(\psi,\Phi) \leq V$. Now observe by inspection that $h(\psi,\Phi)$ is uniformly Lipschitz continuous in $\psi$ and $\Phi$. Thus we can apply Corollary 3.8 of \cite{luo2017decomposition}.
    \end{proof}
\vspace{-0.75cm}
\subsection{\small  Finite Reduction and Algorithmic Solution}
\vspace{-0.2cm}
    The semi-infinite program \eqref{eq:siprog} can be solved via exchange methods \cite{hettich1993semi}, \cite{dong2021wasserstein}, \cite{joachims2009cutting}. 
    We first approximate it by a finite optimization, then iteratively solve this while appending constraints. Let $\tilde{\Gamma} = \{\Phi_1,\dots,\Phi_J\}$ be a collection of $J$ elements in $\Gamma$, i.e., each $\Phi_j, \, j\in[J],$ is a dataset $\{\beta_{t,j}^i, t\in[T]\}_{i\in[\na]}$.  Consider
     the following finite program:
    \vspace{-0.3cm}
    \begin{align}
    \begin{split}
    \label{eq:finred}
        &\min_{\psi\in\Psi,\bv\in\CV} \, \epsilon \cdot v_2 + v_1 \\
        s.t. \, &\max_{\Phi_j \in \tilde{\Gamma}}  G(\psi,\bv, \Phi_j, \hat{\dataset}) \leq 0
         \vspace{-0.4cm}
    \end{split}
    \end{align}
    We can iteratively refine the constraints in the finite program \eqref{eq:finred} by introducing the following maximum constraint violation problem:
    \vspace{-0.2cm}
    \begin{align}
        \begin{split}
        \label{eq:mcv}
            CV = \max_{\Phi \in \Gamma} G(\hat{\psi},\hat{\bv},\Phi,\ndataset) 
        \end{split}
    \end{align}
    where $\hat{\bv} := \{\hat{v}_1, \,\hat{v}_2\}, \hat{\psi} := \{\hu_t^i,\hlam_t^i, t\in[T]\}_{i\in[\na]}$ are optimal solutions to \eqref{eq:finred} under $\tilde{\Gamma}$. Supposing $CV > 0$, we let $\hat{\Phi} \in \Gamma$ be the argument attaining this maximum and append it to $\tilde{\Gamma}$ in \eqref{eq:finred}. Then we iterate, tightening the approximation for the infinite set of constraints in \eqref{eq:siprog} until $CV \leq \delta$; by \cite{dong2021wasserstein} this termination yields a $\delta$-optimal solution of \eqref{eq:siprog}. 
    \setlength{\textfloatsep}{-0.2cm}
    \begin{algorithm}[t]
    \caption{Wasserstein Robust Utility Estimation}
    \label{alg:dro}
    \begin{algorithmic}[1]
        \State Input: Noisy dataset $\ndataset = \{\alpha_t,\,\hb_t^i, t\in[T]\}_{i\in[\na]}$, Wasserstein radius $\epsilon$, stopping tolerance $\delta$. 
        \State Initialize: $\hat{\psi} \in \Psi, \hat{v} \in \CV, \tilde{\Gamma} \leftarrow \emptyset, CV = \delta+1$.
        \While{$CV \geq \delta$}
            \State Solve \eqref{eq:mcv} with $\hat{\psi},\hat{\bv}$, returning $\hat{\Phi}$, $CV$.
            \IIf {$CV$ > 0} $\tilde{\Gamma} \leftarrow \tilde{\Gamma} \cup \hat{\Phi}$ \EndIIf
            \State Solve \eqref{eq:finred} with $\tilde{\Gamma}$, returning $\hat{\psi}, \hat{\bv}$.
        \EndWhile
        \State Output: $\delta$-optimal solution $\hat{\psi}$ of \eqref{eq:siprog}; thus, of \eqref{eq:robest}.
    \end{algorithmic}
    \end{algorithm}
    \setlength{\floatsep}{-0.2cm}

    Algorithm~\ref{alg:dro} illustrates this iterative procedure, and by \cite{dong2021wasserstein} it converges with rate
    $\mathcal{O}\left( \left(\frac{1}{\delta} + 1\right)^{2T\na + 2}\right)$.

\end{section}
\vspace{-0.3cm}
\subsection{\small  Numerical Example}
\label{sec:numeric}
 \vspace{-0.2cm}
The following example is  motivated by the interaction between a cognitive radar network and a target, and can be derived from spectral optimization in this interaction. For brevity we do not expand on this, see \cite{snow2023statistical} for details. We generate the noisy dataset $\ndataset$ \eqref{eq:ndataset} for $\na=3$ agents:
    \begin{align}
    \begin{split}
    &\alpha_t \sim \mathcal{U}(0.1,1.1)^2 \in \reals^2,\, \beta_t^i \in \reals^2, \, t\in \{1,\dots,5\}, \\& \{\beta_t^i\}_{i=1}^3 \in \arg\max_{\{\beta^i\}_{i=1}^{3}}\sum_{i=1}^{3}f^i(\beta^i)\,\, s.t.\, \,\alpha_t'(\sum_{i=1}^{3}\beta^i) \leq 1 \\
    & \hat{\beta}_t^i = \max\{\beta_t^i + \epsilon_t^i\,,0.01(\boldsymbol{1})\},\,\ \epsilon_t^{i} \sim \mathcal{N}(0,1)^2
    \end{split}
    \end{align}
    where  $\boldsymbol{1} =[1, 1]^\prime$, $\max$ operates elementwise, and the utilities of the 3 agents are $f^1(\beta) = \beta(1) + \beta(2), \,f^2(\beta) = \beta(1) + \beta(2)^{1/4},\, f^3(\beta) = \beta(1)^{1/4} + \beta(2)$. We initialize the variables in Algorithm~\ref{alg:dro} as $\delta = 0.1, \, \epsilon = 0.2$. 

  We test the reconstruction accuracy of \eqref{eq:noisyut}, with parameters $\hat{\psi}$ taken from \eqref{eq:parhat} (naive approach) and Algorithm~\ref{alg:dro} (robust approach).  We quantify the reconstruction accuracy in terms of the Hausdorff distance between Pareto-optimal surfaces generated by the reconstructed and true utility functions.\footnote{The reconstruction accuracy of $\{\hat{\util}^i(\cdot)\}_{i=1}^{\na}$ can be quantified as the Hausdorff distance between Pareto-optimal surfaces $E_{f,\alpha}, E_{\hat{\util},\alpha}$, where we define $E_{g,\alpha} = \{x\in\reals^n: x\in \arg\max_{\gamma}\sum_{i=1}^{\na}g^i(\gamma) \, s.t. \, \alpha'\gamma \leq 1\}$.
     This Hausdorff distance is given as $H(E_{f,\alpha},E_{\hat{\util},\alpha})$, given by $H(E_{f,\alpha},E_{\hat{\util},\alpha})  := \max\biggl\{\sup_{x\in E_{f,\alpha}}d(x,E_{\hat{\util},\alpha}), \sup_{y\in E_{\hat{\util},\alpha}}d(y,E_{f,\alpha})\biggr\}$,
    where the distance from point $a$ to set $B$ is $d(a,B) = \inf_{b\in B}d(a,b)$.}
\vspace{-0.2cm}
    \begin{table}[!h]
            \small
            \renewcommand{\arraystretch}{1} 
            \begin{center}
                \begin{tabular}{|p{1cm} | p{2cm} | p{2.5cm}|} 
                    \hline
                     & Average Error & Worst-Case Error \\ \hline
                    Naive  & 0.0627 & 0.9012 \\
                    Robust & 0.0687 & 0.4624 \\
                    \hline
                \end{tabular}
                \vspace{-0.2cm}
               \caption{\small Average and worst-case errors for the naive and robust utility reconstruction procedures, both averaged over 100 Monte-Carlo simulations.}
                \label{tab:1}
            \end{center}
\end{table}
\vspace{-0.5cm}

Table~\ref{tab:1} displays the average error and worst-case error, averaged over 100 Monte-Carlo simulations.  
     
Observe that while Algorithm~\ref{alg:dro} performs similarly to the naive reconstruction on average, its performance is significantly improved in the worst-case. Thus, we verify that Algorithm~\ref{alg:dro} achieves \textit{distributionally robust} utility estimation, \textit{without sacrificing average performance}. The distributional robustness is apparent from the reduced worst-case error,. 

Despite the apparent complexity of the semi-infinite optimization \eqref{eq:siprog}, Figure~\ref{fig:Algcon} shows that a $\delta$-optimal solution from Algorithm~\ref{alg:dro} can be achieved rapidly. Each curve is the average of 100 Monte-Carlo simulations, for different Wassertstein radii $\epsilon$. In each case Algorithm~\ref{alg:dro} produces a $\delta$-optimal solution on average within 10 iterations for $\delta = 0.1$.
\vspace{-0.4cm}
 \begin{figure}[h!]
  \begin{subfigure}{\textwidth}
      \includegraphics[width=0.5\linewidth]{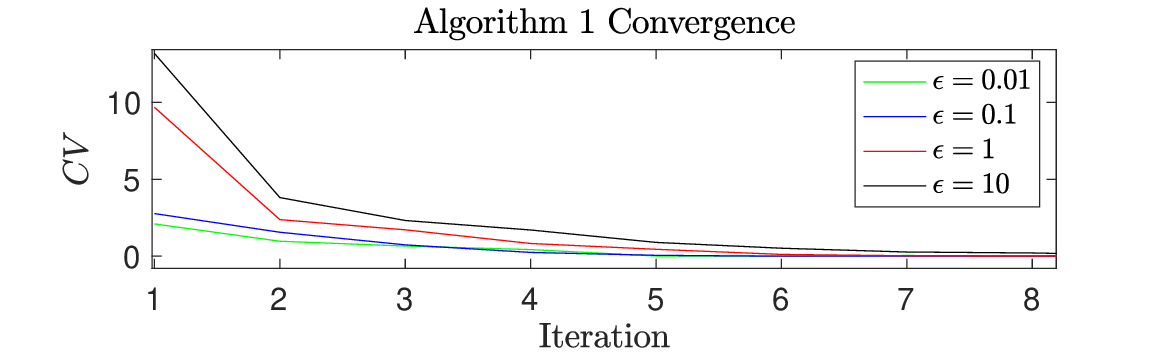}
    \end{subfigure}
    
\vspace{-0.2cm}
    \caption{\small Average convergence of Algorithm~\ref{alg:dro} for varying Wasserstein radii $\epsilon$, over 100 Monte-Carlo simulations.}
    \label{fig:Algcon}
\end{figure}
\vspace{-0.3cm}

\vspace{-0.7cm}
\begin{section}{\small Conclusions}
\vspace{-0.4cm}
We have provided an algorithmic framework for distributionally robust IRL (utility estimation) for coordinated sensing systems. We derived a Wasserstein robust objective using microeconomic revealed preferences, proved its equivalence to a semi-infinite program reformulation, and provided a practical algorithm for obtaining solutions of this reformulation. We illustrated the efficacy of this approach via numerial simulations.
\end{section}

\bibliographystyle{IEEEtran}
\bibliography{Bibliography.bib}

\end{document}